\documentclass[draftclsnofoot, onecolumn,romanappendices, 12pt]{IEEEtran}
\def\paperversion{2} 
\IEEEoverridecommandlockouts

\usepackage{balance}
\usepackage[normalem]{ulem}

\usepackage{algorithm}
\usepackage{algorithmic}
\makeatletter
\def\endthebibliography{%
  \def\@noitemerr{\@latex@warning{Empty `thebibliography' environment}}%
  \endlist
}
\makeatother

\usepackage{cite}
\usepackage{algorithmic}
\usepackage{graphicx}
\usepackage{textcomp}
\usepackage{xcolor}

\usepackage[cmex10]{amsmath}
\usepackage{amssymb}
\usepackage{amsthm}
\usepackage{amsfonts}
\usepackage{mathtools}
\usepackage{bm}
\usepackage{url}
\usepackage{enumerate}
\usepackage{booktabs}
\usepackage{multirow}
\usepackage{colortbl}
\usepackage[font=footnotesize, style=base]{caption}
\usepackage[font=footnotesize]{subcaption}

\usepackage[T1]{fontenc}
\usepackage{float}

\usepackage{tikz}
\usepackage{pgfplots}
\pgfplotsset{my style/.append style={axis x line=middle, axis y line=
middle, xlabel={$x$}, ylabel={$y$}, axis equal }}

\newtheorem{lemma}{Lemma}
\newtheorem{theorem}{Theorem}
\newtheorem{corollary}{Corollary}
\newtheorem{proposition}{Proposition}
\newtheorem{assumption}{Assumption}
\theoremstyle{definition}\newtheorem{definition}{Definition}
\theoremstyle{remark}\newtheorem{example}{Example}
\theoremstyle{remark}\newtheorem{remark}{Remark}

\newcommand{\trans}{^{\mathrm T}}

\newcommand{\diff}{\,\mathrm{d}}
\newcommand{\Lip}{\text{Lip}}
\DeclarePairedDelimiterX{\inp}[2]{\langle}{\rangle}{#1, #2}

\title{
{Two-sample Test using Projected Wasserstein Distance}}
\author{\IEEEauthorblockN{Jie Wang, Rui Gao, and Yao Xie
\thanks{J.~Wang and Y.~Xie are with H. Milton Stewart School of Industrial and Systems Engineering, Georgia Institute of Technology.
R.~Gao is with Department of Information, Risk, and Operations Management, University of Texas at Austin.
Yao Xie is partially supported by an NSF CAREER Award CCF-1650913, DMS-1938106, DMS-1830210, CCF-1442635, and CMMI-1917624.
}
}
}

\begin{document}

\maketitle
\begin{abstract}
We develop a projected Wasserstein distance for the two-sample test, a fundamental problem in statistics and machine learning: given two sets of samples, to determine whether they are from the same distribution. In particular, we aim to circumvent the curse of dimensionality in Wasserstein distance: when the dimension is high, it has diminishing testing power, which is inherently due to the slow concentration property of Wasserstein metrics in the high dimension space. A key contribution is to couple optimal projection to find the low dimensional linear mapping to maximize the Wasserstein distance between projected probability distributions. We characterize theoretical properties of the two-sample convergence rate on IPMs and this new distance.
Numerical examples validate our theoretical results.
\end{abstract}

\section{Introduction}\label{Sec:introduction}
The problem of two-sample testing has been a fundamental topic in statistics and machine learning.
Specifically, one wishes to test whether two collections of samples $x^n:=\{x_i\}_{i=1}^n$ and $y^m:=\{y_i\}_{i=1}^m$ are from the same distribution or not.
Let $\mu$ and $\nu$ denote the underlying unknown distributions for the respective samples. 
A two-sample test is performed to decide whether to accept the null hypothesis $H_0:~\mu=\nu$ or the general alternative hypothesis $H_1:~\mu\ne\nu$.
This problem has applications in a variety of areas.
For instance, in anomaly detection~\cite{Chandola2009, bhuyan2013network, chandola2010anomaly}, the abnormal observations follow a different distribution from the typical distribution.
Similarly, in change-point detection~\cite{Xie13, Shuang15, xie2020sequential}, the post-change observations follow a different distribution from the pre-change one.
Other examples include bioinformatics~\cite{Borgwardt06}, health care~\cite{Schober19}, and statistical model criticism~\cite{Lloyd15, chwialkowski2016kernel, kim2016examples}.

Two-sample testing is a long-standing challenge in statistics.
Classical tests (see, e.g., \cite{lehmann2005testing}) mainly follow the parametric approaches, which are designed based on prior information about the distributions under each class.
Examples in classical tests include the Hotelling's two-sample test~\cite{hotelling1931} and the Student's t-test~\cite{PFANZAGL96}.
In this paper, we consider non-parametric two-sample testing, in which no prior information about the unknown distribution is available.
Two-sample tests for non-parametric settings are usually constructed based on some metrics quantifying the distance between two distributions.
Some earlier work designs two-sample tests based on the Kolmogorov-Smirnov distance~\cite{Pratt1981, Frank51}, the total variation distance~\cite{Ga1991}, and the Wasserstein distance~\cite{delbarrio1999, ramdas2015wasserstein}.
Those approaches work well for univariate two-sample tests, but they do not generalize for high-dimensional observations. 

Several data-efficient two-sample tests~\cite{Gretton12, Gretton09, Grettonnips12} are constructed based on Maximum Mean Discrepancy~(MMD), which quantifies the distance between two distributions by introducing test functions in a Reproducing Kernel Hilbert Space~(RKHS).
However, it is pointed out in \cite{reddi2014decreasing} that when the bandwidth is chosen based on the median heuristic, the MMD tests suffer from decaying power in high dimensions.
Some other probability distances such as Sinkhorn divergence~\cite{Bigot2017CentralLT}, $f$-divergence~\cite{kanamori2010fdivergence}, and classifier-based distances~\cite{kim2020classification} are considered for more efficient two-sample tests in high-dimensional settings.


In this paper, we analyze non-parametric two-sample tests based on general \emph{integral probability metrics}~(IPMs). 
Many existing tests such as \cite{Gretton12, Ga1991, ramdas2015wasserstein} are constructed based on IPMs, which can be viewed as special cases in our unified framework.
The quality of two-sample tests based on IPM depends on the choice of the function space.
On the one hand, it should be rich enough to claim $\mu=\nu$ if the metric vanishes.
On the other hand, to control the type-I error, the function space should also be relatively small so that the empirical estimate of IPM decays quickly into zero. 
The Wasserstein distance, as a particular case of IPM, is popular in many machine learning applications. However, a significant challenge in utilizing the Wasserstein distance for two-sample tests is that the empirical Wasserstein distance converges at a slow rate due to the complexity of the associated function space.
Thus, its performance suffers from the {\it curse of dimensionality}.

We summarize the contributions of this work as follows. 
\ifnum\paperversion=1
A full version of this paper is accessible at \cite{wang2021twosampleoriginal}.
\fi
\begin{itemize}
    \item 
The finite-sample convergence of general IPMs based on empirical samples is discussed based on the Rademacher complexity argument.
\item
The projected Wasserstein distance is developed to improve the convergence rate of the empirical Wasserstein distance.
Two-sample tests based on this new metric are proposed because of its satisfactory statistical property.
\item
Numerical experiments show that the two-sample test using the projected Wasserstein distance has comparable performance with existing state-of-the-art methods.
\end{itemize}


\subsection{Related Work}

Low-dimensional projections are commonly used for understanding the structure of high-dimension distributions.
Typical examples include principal component analysis~\cite{Jolliffe1986}, linear discriminant analysis~\cite{mclachlan1992discriminant}, etc.
It is intuitive to understand the differences between two collections of high-dimensional samples by projecting those samples into low-dimensional spaces in some proper directions~\cite{wei2013directionprojectionpermutation, Ghosh16, mueller2015principal, xie2020sequential, lin2020projection, lin2020projection2, huang2021riemannian}.
\cite{wei2013directionprojectionpermutation, Ghosh16} design the direction by training binary linear classifiers on samples.
\cite{mueller2015principal, xie2020sequential} find the worst-case direction that maximizes the Wasserstein distance between projected sample points in one-dimension.
Recently, \cite{lin2020projection, lin2020projection2, huang2021riemannian} naturally extend this idea by projecting data points into a $k$-dimensional linear subspace with $k>1$ such that the $2$-Wasserstein distance after projection is maximized.
Our proposed projected Wasserstein distance is similar to this framework, but we use $1$-Wasserstein distance instead.
By leveraging tools from generalization bounds in IPMs, we are able to give a tighter finite-sample convergence rate of the proposed distance.

In high-dimensional settings, the Wasserstein distance is difficult to compute, and its convergence rate is slow.
Several variants of the Wasserstein distance have been developed in the literature to address these two issues.
The smoothed Wasserstein distance is designed to reduce the computational cost~\cite{cuturi2013sinkhorn} and improve the sample complexity~\cite{genevay2019sample} by using entropic regularizations.
Some projection-based variants of the Wasserstein distance are also discussed to address the computational complexity issue, including the sliced~\cite{Bonneel14} and the max-sliced~\cite{deshp2019maxsliced} Wasserstein distances. 
Sliced Wasserstein distance is based on the average Wasserstein distance between two projected distributions along infinitely many random one-dimensional linear projection mappings.
It is shown in \cite{manole2019minimax} that its empirical estimate decays into zero with rate $O(n^{-1/2})$ under mild conditions, and a two-sample test can be constructed based on this nice statistical behavior.
However, it is costly to compute this distance because a large number of random projection mappings are required to approximate the distance within minimal precision error.
The max-sliced Wasserstein distance is proposed to address this issue by finding the worst-case one-dimensional projection mapping such that the Wasserstein distance between projected distributions is maximized.
The projected Wasserstein distance proposed in our paper generalizes the max-sliced Wasserstein distance by considering the $k$-dimensional projection mappings, and we discuss the finite-sample convergence rate of the projected Wasserstein distance so that two-sample tests can be designed.

\section{Problem Setup}\label{Sec:problem}
Let $x^n:=\{x_i\}_{i=1}^n$ and $y^m:=\{y_i\}_{i=1}^m$ be i.i.d. samples generated from distributions $\mu$ and $\nu$, respectively.
The supports of $\mu$ and $\nu$ are denoted as $\text{supp}(\mu)$ and $\text{supp}(\nu)$, respectively.
We assume that both $\mu$ and $\nu$ are unknown distributions, and the supports of them belong to the metric space $(\mathbb{R}^d, \textsf{d})$ with $\textsf{d}$ being the Euclidean metric.
A two-sample test is performed to decide whether accept $H_0:~\mu=\nu$ or $H_1:~\mu\ne\nu$ based on collected samples.
Denote by $T:~(x^n, y^m)\to\{d_0, d_1\}$ the two-sample test, where $d_i$ means that we accept the hypothesis $H_i$ and reject the other, $i=0,1$.
The type-I and type-II error probabilities for the test $T$ are defined as
\begin{align*}
\epsilon^{(\text{I})}_{n,m}&={\mathbb{P}}_{x^n\sim\mu, y^m\sim\nu}
\bigg(
T(x^n,y^m)=d_1
\bigg),\quad\text{under }H_0,\\
\epsilon^{(\text{II})}_{n,m}&={\mathbb{P}}_{x^n\sim\mu, y^m\sim\nu}
\bigg(
T(x^n,y^m)=d_0
\bigg),\quad\text{under }H_1.
\end{align*}
Denote by $\hat{\mu}_n$ and $\hat{\nu}_m$ the empirical distributions constructed from i.i.d. samples from $\mu$ and $\nu$:
\[
\hat{\mu}_n\triangleq \frac{1}{n}\sum_{i=1}^n\delta_{x_i}, \qquad
\hat{\nu}_m\triangleq \frac{1}{m}\sum_{i=1}^m\delta_{y_i}.
\]
Given collected samples $x^n$ and $y^m$, a non-parametric two-sample test is usually constructed based on IPMs, which quantify the discrepancy between the associated empirical distributions.
\begin{definition}[Integral Probability Metric]
Given two distributions ${\mu}$ and ${\nu}$, define the integral probability metric as
\[
\text{IPM}({\mu},{\nu})=\sup_{f\in\mathcal{F}}\bigg(
\int f(x)\diff{\mu}(x)
-
\int f(y)\diff{\nu}(y)
\bigg).
\]
\end{definition}

Denote by $W$ the $1$-Wasserstein distance for quantifying the distance between two probability distributions:
\[
W(\mu,\nu)
=
\begin{aligned}
\min_{\pi\in\Pi(\mu,\nu)}&~\int c(x,y)\diff \pi(x,y),
\end{aligned}
\]
where $\Pi(\mu,\nu)$ denotes the set of joint distributions whose marginal distributions are $\mu$ and $\nu$.
We focus on the case where the cost function $c(\cdot,\cdot)=\textsf{d}(\cdot,\cdot)$. 
By the Kantorovich-Rubinstein duality result~\cite[Theorem~5.9]{villani2016optimal}, the Wasserstein distance can be reformulated as a special integral probability metric:
\[
\begin{aligned}
W(\mu,\nu)&=\sup_{f\in\text{Lip}_1}\left(\int f(x)\diff\mu(x) - \int f(y)\diff\nu(y)\right),
\end{aligned}
\]
where the space $\text{Lip}_1$ denotes a collection of $1$-Lipschitz functions:
\[
\text{Lip}_1=\left\{f:~
\sup_{
\substack{
x\in\text{supp}(\mu),
y\in\text{supp}(\nu),
\\
x\ne y
}
}
\frac{|f(x) - f(y)|}{c(x,y)}
\le 1\right\}.
\]
While the Wasserstein distance has wide applications in machine learning, the finite-sample convergence rate of the Wasserstein distance between empirical distributions is slow in high-dimensional settings.
We propose the projected Wasserstein distance to address this issue.
\begin{definition}[Projected Wasserstein Distance]\label{Definition:projected:wasserstein:distance}
Given two distributions ${\mu}$ and ${\nu}$, define the projected Wasserstein distance as
\[
\mathcal{P}W(\mu, \nu)
=
\max_{\mathcal{A}:~\mathbb{R}^d\to\mathbb{R}^k}~
W\left(
\mathcal{A}\#\mu,
\mathcal{A}\#\nu
\right)\quad
\mathrm{s.t.}\quad
A\trans A=I_k,
\]
where the operator $\#$ denotes the push-forward operator, i.e.,
\[
\mathcal{A}(z)\sim \mathcal{A}\#\mu\quad\text{for }z\sim\mu,
\]
and we denote $\mathcal{A}$ as a linear operator such that $\mathcal{A}(z)=A\trans z$ with $z\in\mathbb{R}^d$ and $A\in\mathbb{R}^{d\times k}$.
\end{definition}
The orthogonal constraint on the projection mapping $A$ is for normalization, such that any two different projection mappings have distinct projection directions.
The projected Wasserstein distance can also be viewed as a special case of integral probability metric with the function space
\begin{equation}\label{Eq:projected:function:space}
\mathcal{F}=\{z\mapsto g(A\trans z):~g\in\text{Lip}_1, A\trans A=I_k\}.
\end{equation}

In this paper, we design the two-sample test so that the acceptance region for the null hypothesis $H_0$ is given by:
\begin{equation}\label{Eq:acceptance:region}
\left\{
(x^n, y^m):~
\mathcal{P}W(\hat{\mu}_n, \hat{\nu}_m)
\le\gamma_{n,m}
\right\},
\end{equation}
where the hyper-parameter $\gamma_{n,m}$ is chosen so that the type-I error $\epsilon_{n,m}^{(\text{I})}$ is approximately bounded by $\alpha$ and the type-II error $\epsilon_{n,m}^{(\text{II})}$ is as small as possible.

Our two-sample testing algorithm also gives us interpretable characterizations for understanding differences between two high-dimensional distributions, by studying the worst-case projection mappings and projected samples in low dimensions.
See Fig.~\ref{fig:sample:complexity:a}) for the optimized linear mapping so that the Wasserstein distance between the projected empirical distributions is maximized, and Fig.~\ref{fig:sample:complexity:b}) for the illustration of the kernel density estimate~(KDE) plot for two projected empirical distributions in one dimension.

\section{Two-sample Testing using IPM}\label{Sec:inference}
Let the two-sample hypothesis testing be the following:
\[
H_0:~\mu=\nu,\qquad
H_1:~\mu\ne\nu.
\]
In this section, we first discuss the finite-sample guarantee for general IPMs, then a two-sample test can be designed based on this statistical property. Finally, we design a two-sample test based on the projected Wasserstein distance. 
\ifnum\paperversion=2
Omitted proofs can be found in Appendix~\ref{Sec:appendix:proof}.
\fi
The finite-sample guarantee relies on the notion of Rademacher complexity defined as follows.
\begin{definition}[Rademacher complexity]
Given the function space $\mathcal{F}$ and a distribution ${\mu}$, define the Rademacher complexity as
\[
\mathfrak{R}_n(\mathcal{F}, {\mu})
:=
\mathbb{E}_{x_i\sim {\mu}, \sigma_i}
\left[
\sup_{f\in\mathcal{F}}
\frac{1}{n}\sum_{i=1}^n\sigma_if(x_i)
\right],
\]
where $\sigma_i$'s are i.i.d. and take values in $\{-1,1\}$ with equal probability.
\end{definition}

Throughout this section, we make two technical assumptions that are standard in the literature.
\begin{assumption}\label{Assumption:light:tail}
\begin{enumerate}
\item[\text{(I)}.]
Any function $f\in\mathcal{F}$ is $L$-Lipschitz:
\[
|f(x) - f(x')|\le L\textsf{d}(x,x').
\]
\item[\text{(II)}]
The supports of target distributions $\mu$ and $\nu$ have finite diameters, $B_{\mu}$ and $B_{\nu}$, respectively:
\begin{align*}
    \sup_{x,x'\in\text{supp}(\mu)}\textsf{d}(x,x')&\le B_{\mu},
    \sup_{y,y'\in\text{supp}(\nu)}\textsf{d}(y,y')\le B_{\nu}.
\end{align*}
\end{enumerate}
\end{assumption}

Assumption~\ref{Assumption:light:tail}(II) does not hold when distributions $\mu$ and $\nu$ have unbounded supports.
In that case, we restrict the target distribution in a bounded support such that the probability of locating in such support is relatively large.
However, the diameter cannot be chosen arbitrarily large since otherwise the sample complexity bound will become too conservative.
For instance, when the distribution $\mu$ is known to be sub-Gaussian with parameter $\sigma$, we restrict the support to be $\big(\mathbb{E}_{\mu}[X] - \sqrt{2\log(1/\eta)}\sigma, \mathbb{E}_{\mu}[X] + \sqrt{2\log(1/\eta)}\sigma\big)$, which implies the probability that the distribution locates in this interval is at least $1-\eta$.




\begin{proposition}[Finite-sample Guarantee for the rate of convergence of IPM]\label{Proposition:finite:IPM}
Assume that Assumption~\ref{Assumption:light:tail} is satisfied,
and let $\epsilon>0$.
Then with probability at least,
\begin{equation}
\label{Eq:failure}
1-2\exp\left(
-\frac{2\epsilon^2mn}{L^2(mB_{\mu}^2 + nB_{\nu}^2)}
\right),
\end{equation}
we have
\begin{equation}
\label{Eq:concentration}
\big|\text{IPM}({\mu},{\nu}) - \text{IPM}(\hat{{\mu}}_n,\hat{{\nu}}_m)\big|
<
\epsilon
+
2[
\mathfrak{R}_n(\mathcal{F}, {\mu})
+
\mathfrak{R}_m(\mathcal{F}, {\nu})
].
\end{equation}
\end{proposition}


Proposition~\ref{Proposition:finite:IPM} establishes the finite-sample guarantee for the convergence of empirical two-sample IPMs. 
We can build two-sample tests by leveraging this theoretical result.
Specifically, the failure probability in \eqref{Eq:failure} controls the level of type-I error,
and the right-hand side in the inequality \eqref{Eq:concentration} controls the type-II error.

The proof of Proposition~\ref{Proposition:finite:IPM} essentially follows the one-sample generalization bound mentioned in \cite[Theorem~3.1]{zhang2018discriminationgeneralization}.
However, by following the similar proof procedure discussed in \cite{Gretton12}, we can improve this two-sample finite-sample convergence result when extra assumptions hold, but existing works about IPMs haven't investigated it yet. 

\begin{proposition}[Improved rate of convergence of IPM]
\label{Proposition:special:P1}
Under the setting in Proposition~\ref{Proposition:finite:IPM}, assume further that the sample size $n=m$ and distributions $\mu=\nu$, then with probability at least 
\[
1-2\exp\left(
-\frac{\epsilon^2n}{L^2B_{\mu}^2}
\right),
\]
we have
\[
\big|\text{IPM}({\mu},{\nu}) - \text{IPM}(\hat{{\mu}}_n,\hat{{\nu}}_m)\big|
<
\epsilon
+
2
\mathfrak{R}_n(\mathcal{F}, {\mu}).
\]
\end{proposition}

\begin{corollary}\label{Corollary:IPM:accept}
Under the setting in Proposition~\ref{Proposition:finite:IPM}, assume that $\text{IPM}({\mu},{\nu})=0$ if and only if ${\mu}={\nu}$.
\begin{enumerate}
\item
A hypothesis test of level $\alpha$ for null hypothesis ${\mu}={\nu}$,
has the acceptance region
\begin{align*}
&\text{IPM}(\hat{{\mu}}_n,\hat{{\nu}}_m)
<\sqrt{\frac{L^2B_\mu^2(m+n)}{2mn}\log\frac{2}{\alpha}}
\\
&\qquad+
2[
\mathfrak{R}_n(\mathcal{F}, {\mu})
+
\mathfrak{R}_m(\mathcal{F}, {\mu})
].
\end{align*}
\item
Assume further that $m=n$.
Then a hypothesis test of level $\alpha$ for null hypothesis ${\mu}={\nu}$,
has the acceptance region
\begin{align*}
&\text{IPM}(\hat{{\mu}}_n,\hat{{\nu}}_m)
<\sqrt{\frac{L^2B_\mu^2}{n}\log\frac{2}{\alpha}}+2\mathfrak{R}_n(\mathcal{F}, {\mu}).
\end{align*}
\end{enumerate}
\end{corollary}

To obtain reliable performance on two-sample testing, the target function space should balance the following trade-off:
\begin{enumerate}
\item 
The function space $\mathcal{F}$ should be rich enough such that $\text{IPM}(\cdot,\cdot)$ is a well-defined distance;
\item
According to Corollary~\ref{Corollary:IPM:accept}, the Rademacher complexity of the function space $\mathcal{F}$ should be relatively small so that the empirical IPM decays at a fast rate when the sample size increases.
\end{enumerate}

\begin{example}[Wasserstein Distance for Two-sample Tests]\label{Wasserstein:two:sample}
The $1$-Wasserstein distance can be viewed as a special IPM with $\mathcal{F}=\text{Lip}_1$, where the Rademacher complexity of $\mathcal{F}$ is given by \cite[Example~4]{Luxburg04}:
\begin{align*}
&\mathfrak{R}_n(\mathcal{F}, {\mu})
=\left\{
\begin{aligned}
\mathcal{O}\left(\sqrt{\log(n)}\cdot n^{-1/2}\right), &\quad\text{if $d\le 2$},\\
\mathcal{O}\left(n^{-1/d}\right), &\quad\text{if $d\ge 3$}.
\end{aligned}
\right.
\end{align*}
As a result, the sample complexity for estimating the Wasserstein distance $W(\mu,\nu)$ up to $\epsilon$ sub-optimality gap is of order $\tilde{\mathcal{O}}(\epsilon^{d\lor 2})$.
This suggests that using the Wasserstein distance for two-sample tests suffers from the curse-of-dimensionality.
\end{example}

Motivated by Example~\ref{Wasserstein:two:sample}, we propose the projected Wasserstein distance in Definition~\ref{Definition:projected:wasserstein:distance} to improve the sample complexity. 
This distance can be viewed as a special IPM with the function space defined in \eqref{Eq:projected:function:space}, a collection of $1$-Lipschitz functions in composition with an orthogonal $k$-dimensional linear mapping.
We first investigate its finite-sample guarantee by upper bounding the Rademacher complexity for $\mathcal{F}$ defined in \eqref{Eq:projected:function:space}.

\begin{proposition}
\label{Proposition:Rademacher}
Let $\mathcal{F}$ be the function space defined in \eqref{Eq:projected:function:space}.
Then
\[
\mathfrak{R}_n(\mathcal{F}, {\mu}) \le \sqrt{\frac{2k}{n} \mathbb{E}_{{\mu}}[\|X\|^2]} + \mathcal{J}_n,
\quad\text{where }
\mathcal{J}_n \triangleq \inf_{\epsilon>0}
\left\{
2\epsilon + \sqrt{\frac{36}{n}(B_{\mu}^2+4D_{\mu}^2)\log\mathcal{N}}
\right\},
\]
with $D_{\mu}=\max_{x}\|x\|$ and $\mathcal{N}$ denoting the covering number of the unit ball of $1$-Lipschitz functions:
\[
\mathcal{N}=\left( 
2\left\lceil 
\frac{2B_{\mu}}{\epsilon}
\right\rceil+1
\right)\cdot 2^{(1+2B_{\mu}/\epsilon)^k}.
\]
\end{proposition}


We can develop a two-sample test using the projected Wasserstein distance based on Corollary~\ref{Corollary:IPM:accept} and Proposition~\ref{Proposition:Rademacher}.
Except for the second-order moment term, the acceptance region does not depend on the dimension of the support of distributions, but only on the sample size and the dimension of projected spaces.
\begin{figure*}[t]
\centering
\subcaptionbox{\label{fig:testing:AUC:a}}{\includegraphics[width=0.4\textwidth]{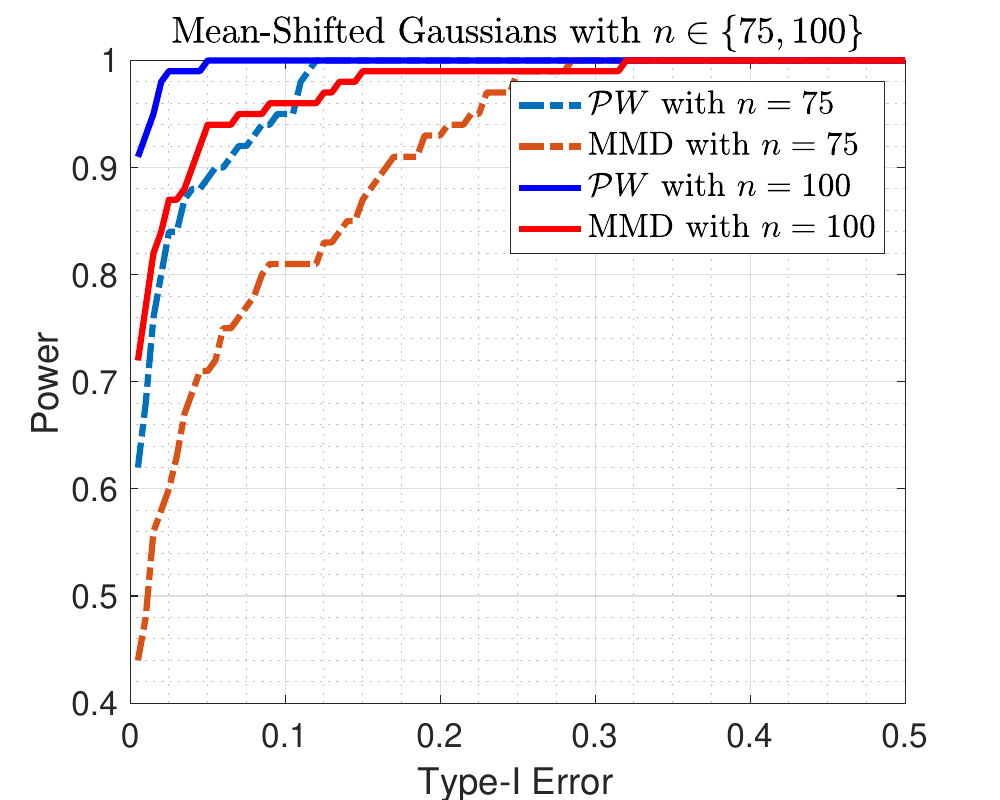}}\hfill
\subcaptionbox{\label{fig:testing:AUC:b}}{\includegraphics[width=0.4\textwidth]{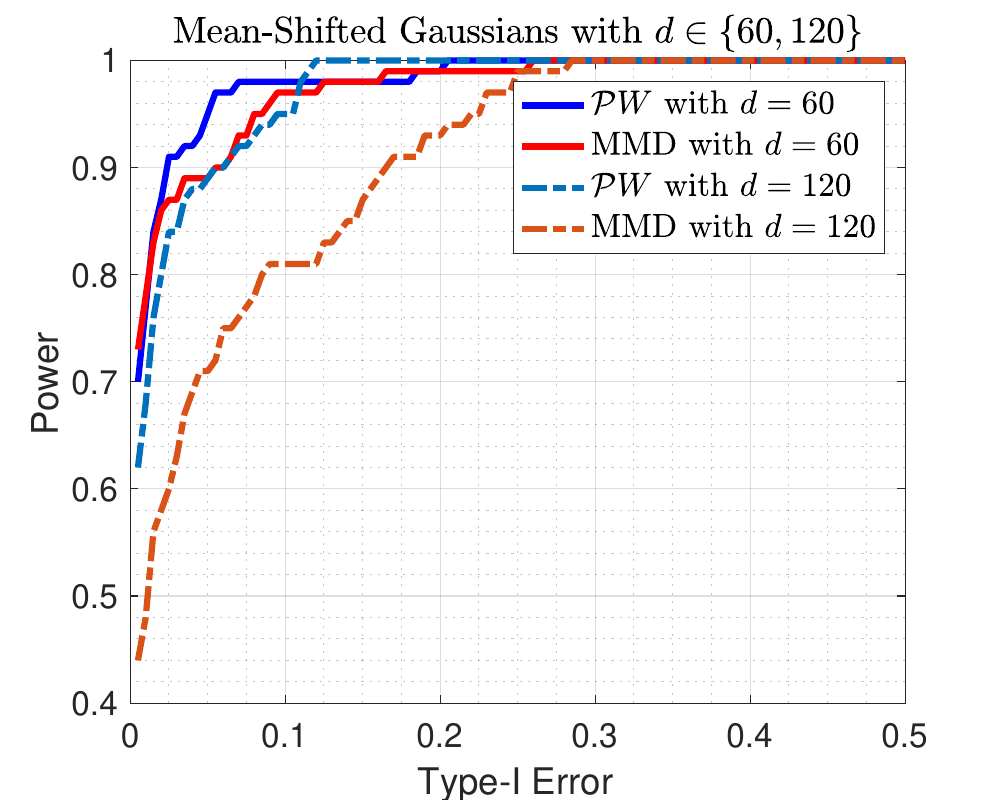}}\hfill
\subcaptionbox{\label{fig:testing:AUC:c}}{\includegraphics[width=0.4\textwidth]{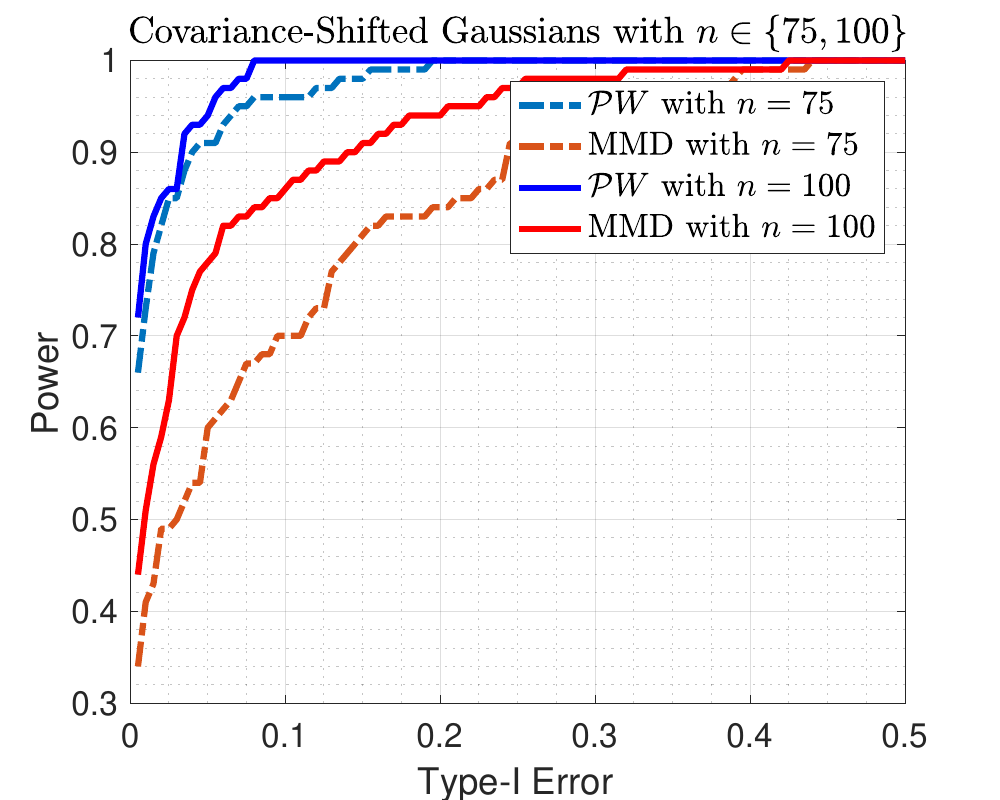}}\hfill
\subcaptionbox{\label{fig:testing:AUC:d}}{\includegraphics[width=0.4\textwidth]{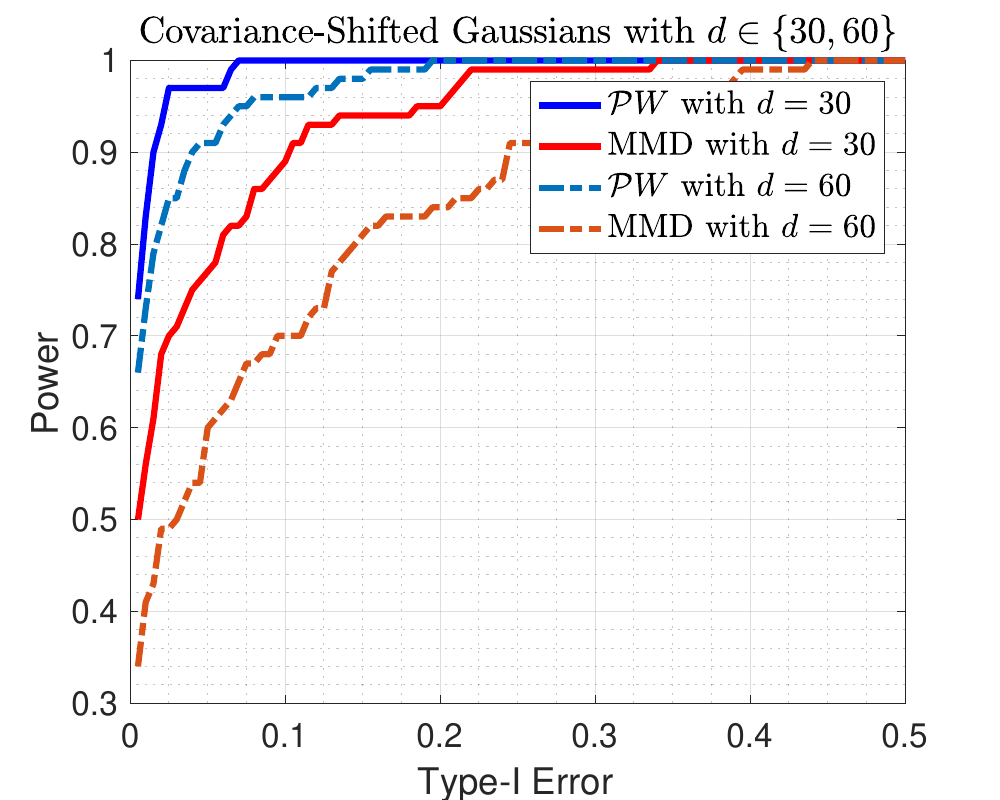}}
\caption{
Comparison of ROC curves of the PW test versus the MMD test on two types of synthesis data.
Fig.~\ref{fig:testing:AUC:a}): Mean-shifted Gaussians with $n\in\{75,100\}$ and fixed $d=120$.
Fig.~\ref{fig:testing:AUC:b}): Mean-shifted Gaussians with $d\in\{60,120\}$ and fixed $n=75$.
Fig.~\ref{fig:testing:AUC:c}): Covariance-shifted Gaussians with $n\in\{75,100\}$ and fixed $d=60$.
Fig.~\ref{fig:testing:AUC:d}): Covariance-shifted Gaussians with $d\in\{30,60\}$ and fixed $n=75$.
}
\label{fig:testing:AUC}
\end{figure*}

\begin{theorem}\label{Theorem:Sliced:accept}
Assume that Assumption~\ref{Assumption:light:tail}(II) is satisfied.
The following properties for two-sample tests using projected Wasserstein distance hold:
\begin{enumerate}
\item
A hypothesis test of level $\alpha$ for null hypothesis ${\mu}={\nu}$, has the acceptance region
\begin{align*}
&\mathcal{P}W(\hat{\mu}_n, \hat{\nu}_n)
\le 
\sqrt{\frac{B_\mu^2(m+n)}{2mn}\log\frac{2}{\alpha}}
+
2\left[
\sqrt{\frac{2k}{n}\mathbb{E}_{{\mu}}[\|X\|^2]}
+
\sqrt{\frac{2k}{m}\mathbb{E}_{{\mu}}[\|X\|^2]}
+
\mathcal{J}_n+\mathcal{J}_m
\right].
\end{align*}
\item
Assume further that $m=n$. 
Then a hypothesis test of level $\alpha$ for null hypothesis ${\mu}={\nu}$, has the acceptance region
\begin{align*}
&\mathcal{P}W(\hat{\mu}_n, \hat{\nu}_n)
\le 
\frac{B_{\mu}\sqrt{\log\frac{2}{\alpha}}}{\sqrt{n}}
+2\sqrt{\frac{2k}{n}\mathbb{E}_{{\mu}}[\|X\|^2]}+2\mathcal{J}_n.
\end{align*}
\item
Provided that $\mathcal{P}W(\mu,\nu)>\gamma_{m,n}$, then the type-II error for two-sample tests with acceptance region defined in \eqref{Eq:acceptance:region} has the upper bound:
\[
\text{Pr}_{\mathcal{H}_1}\bigg(
\mathcal{P}W(\hat{\mu}_n, \hat{\nu}_m)
<\gamma_{m,n}
\bigg)
\le \frac{\text{MSE}(\mu,\nu)}{\left(
\mathcal{P}W(\mu,\nu) - \gamma_{m,n}
\right)^2},
\]
where $\text{MSE}(\mu,\nu):=\mathbb{E}\left[\big(
\mathcal{P}W(\hat{\mu}_n, \hat{\nu}_m)
-
\mathcal{P}W({\mu}, {\nu})
\big)^2\right].$
\end{enumerate}
\end{theorem}
Since the second-order moment terms in the threshold are unknown, we can replace it with the unbiased estimate based on collected samples when performing the two-sample test.
The term $\mathcal{J}_n$ dominates the threshold since it scales with order $\tilde{O}(n^{-1/(k\lor 2)})$ due to the argument based on the Rademacher complexity. 
In order to maintain satisfactory performance on two-sample tests, the parameter $k$ should be chosen to be relatively small in practice.

\begin{remark}
Existing works including \cite{mueller2015principal, lin2020projection} also talk about the sample complexity bounds for the projected Wasserstein distance.
However, the bound presented in \cite{mueller2015principal} depends on the input dimension $d$ and focuses on the case $k=1$ only.
\cite{lin2020projection} slightly improves Assumption~\ref{Assumption:light:tail}(II) into light tail conditions, but constants presented in the sample complexity bound are not characterized explicitly,
which makes it impractical for two-sample tests.

\end{remark}

\section{Numerical Experiment}\label{Sec:numerical}
\begin{figure}
\subcaptionbox{\label{fig:sample:complexity:a}}{\includegraphics[width=0.4\textwidth]{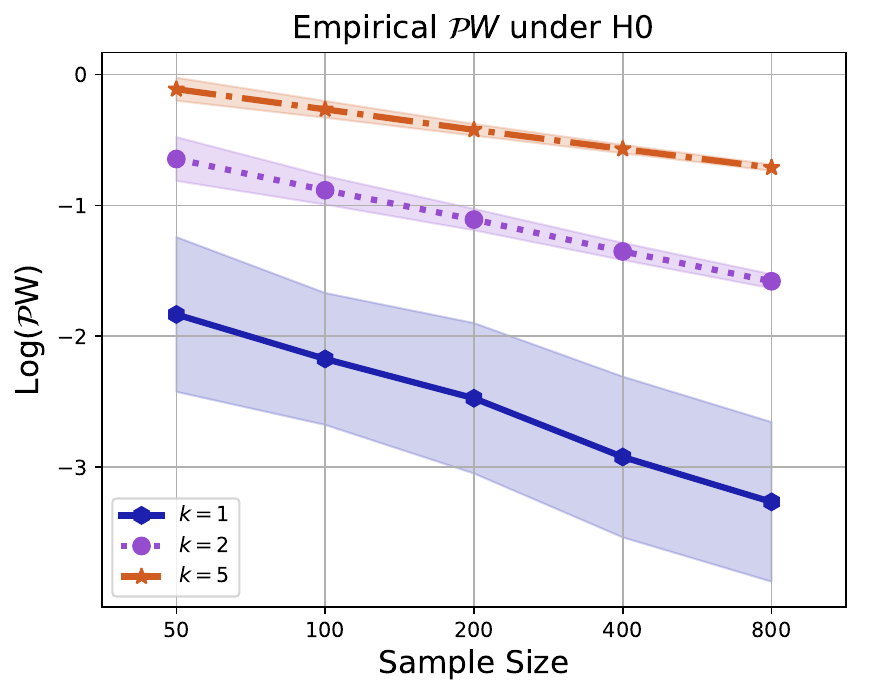}}\hfill
\subcaptionbox{\label{fig:sample:complexity:b}}{\includegraphics[width=0.4\textwidth]{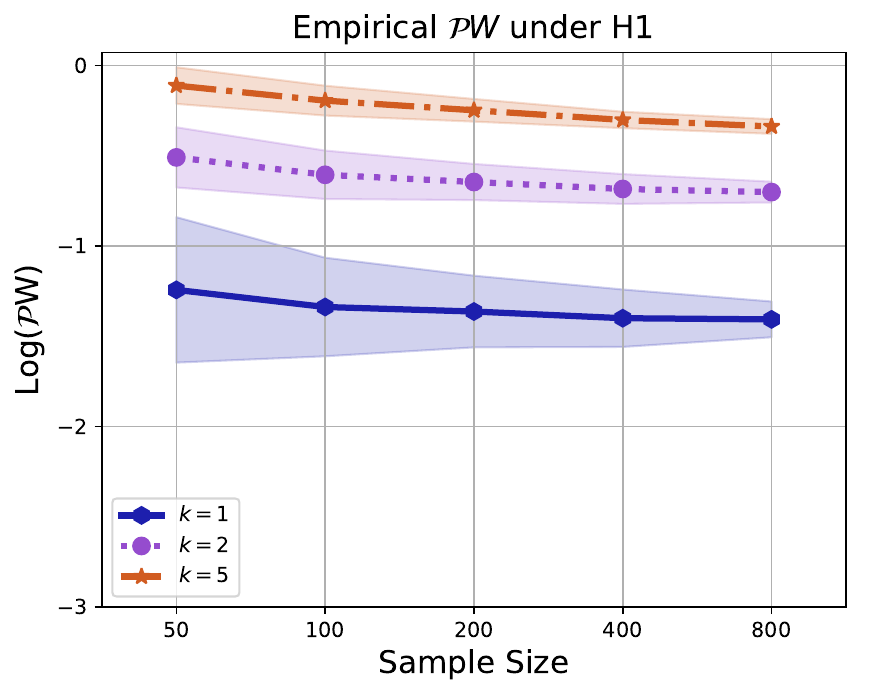}}
\caption{Mean values and $95\%$-confidence intervals for $\mathcal{P}W(\hat{\mu}_n, \hat{\nu}_n)$ across different numbers of samples $n$. Results are averaged over $100$ independent trials. 
Fig.~\ref{fig:sample:complexity:a}) corresponds to $H_0$ and Fig.~\ref{fig:sample:complexity:b}) corresponds to $H_1$.
}\label{fig:sample:complexity}
\end{figure}

\begin{figure}
\subcaptionbox{\label{fig:visualization:a}}{\includegraphics[width=0.4\textwidth]{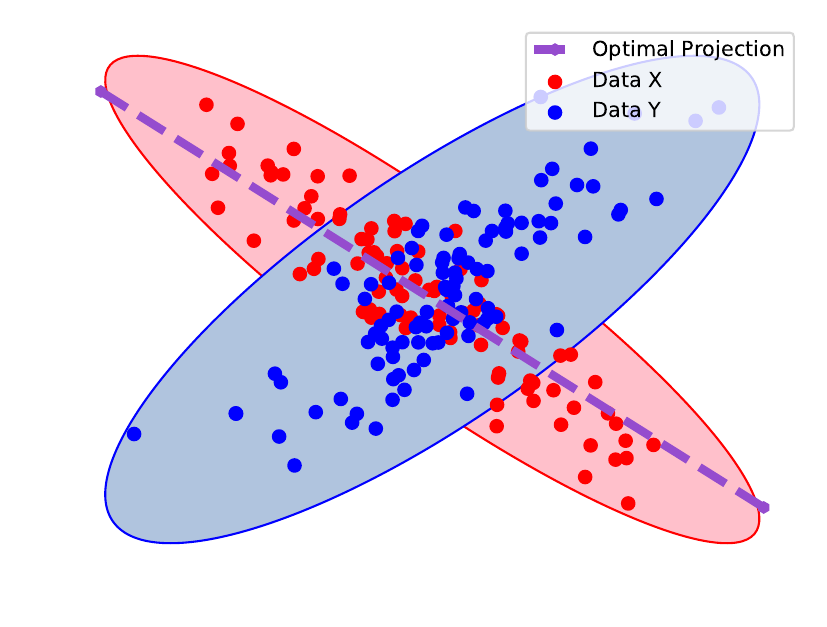}}\hfill
\subcaptionbox{\label{fig:visualization:b}}{\includegraphics[width=0.4\textwidth]{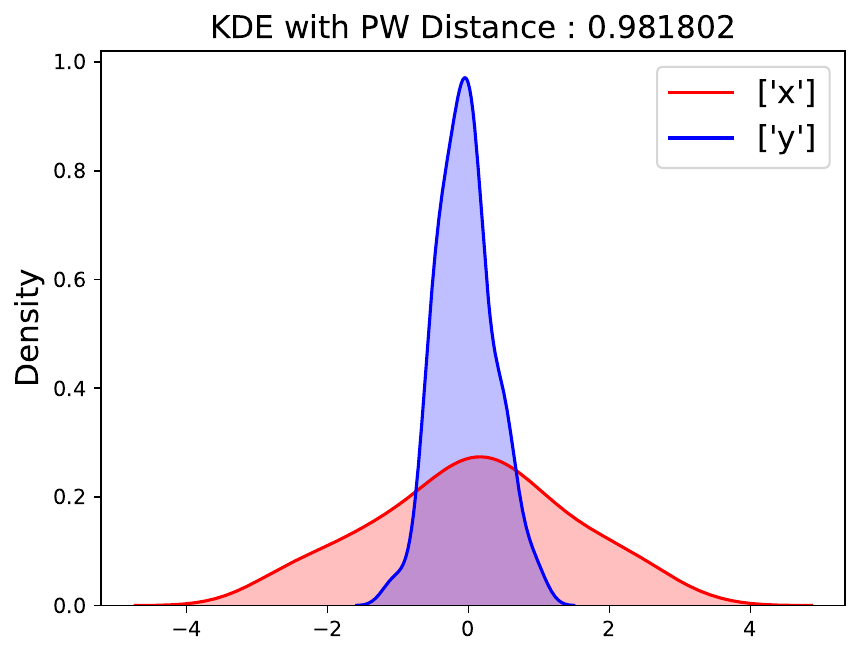}}
\caption{
\ref{fig:visualization:a}) Illustration of the projection mapping trained on two collections of samples generated from two different target distributions with $m=n=100$.
Here the red and blue points are generated from Gaussian distributions with two different covariance matrix. 
The purple arrow denotes the optimized projection mapping.
\ref{fig:visualization:b}) KDE plot for the empirical distributions after projection.
}
\label{fig:visualization}
\end{figure}

In this section, we present several numerical examples to validate our theory and demonstrate the performance of the proposed two-sample tests.
The computation of projected Wasserstein distance was recently studied in \cite{Subspacepaty, lin2020projection, huang2021riemannian}.
\ifnum\paperversion=2
We use the Riemannian gradient method discussed in \cite[Algorithm~3]{lin2020projection} to compute the projected Wasserstein distance, where the details of the corresponding algorithm are summarized in Appendix~\ref{Appendix:sec:computation}.
\fi
\ifnum\paperversion=1
We use the Riemannian gradient method discussed in \cite[Algorithm~3]{lin2020projection} to compute the projected Wasserstein distance.
\fi

However, the two-sample tests based on concentration inequalities in Section~\ref{Sec:inference} give conservative results in practice. We examine the two-sample tests using the projected Wasserstein distance via the permutation approach.
Specifically, we permute the collected data points for $N_p=100$ times, and the $p$-value of the proposed test can be computed as the fraction of times that the projected Wasserstein distances under permuted samples are greater than the projected Wasserstein distance under the original empirical samples.


\subsection{Sample Complexity of $\mathcal{P}W$}
First, we test the numerical convergence of empirical projected Wasserstein distance $\mathcal{P}W(\hat{\mu}_n,\hat{\nu}_n)$ under $H_0$ and $H_1$ across different sample sizes $n$ and dimensions of projected space $k$.
We take $d=30$ in this experiment.
When under $H_0$, we set target distributions $\mu$ and $\nu$ to be the uniform distribution on $[-1,1]^d$.
When under $H_1$, we set the distribution $\mu$ to be the uniform distribution on $[-1,1]^d$, and $\nu$ to be the Gaussian distribution $\mathcal{N}(0,\sigma^2I_d)$ truncated on the interval $[-1,1]^d$ with $\sigma=\frac{1}{1.96}$.
Fig.~\ref{fig:sample:complexity} presents the behavior of $\mathcal{P}W(\hat{\mu}_n,\hat{\nu}_n)$ under $H_0$ and $H_1$, respectively, where for fixed $n$ and $k$ the results are averaged for $100$ independent trials.
From the plot we can see that when under $H_0$, the empirical estimators decay quickly to zero, but the convergence rate slows down when the dimension of projected space $k$ increases.
Moreover, the estimators are bounded away from zero under $H_1$.
These observations justify the finite-sample guarantee as discussed before.
\subsection{Visualization for Two-sample Testing}\label{sec:visual}

Our two-sample testing framework provides a visual interpretation for classifying the differences between two collected samples.
Fig.~\ref{fig:visualization}a) present the scatter plots of two collections of data points together with the optimal projection mapping.
Specifically, the red points in Fig.~\ref{fig:visualization}a) are sampled i.i.d. from $\mathcal{N}(0,\Sigma_X)$ and blue points are sampled from $\mathcal{N}(0,\Sigma_Y)$, where
\[
\Sigma_X=
\begin{pmatrix}
1&-0.9\\-0.9&1
\end{pmatrix},\quad \Sigma_Y=\begin{pmatrix}
1&0.8\\0.8&1
\end{pmatrix}.
\]
The sample size is $m=n=100$. 
Fig.~\ref{fig:visualization}b) demonstrates the KDE plot of empirical distributions after optimal projection, from which we can identify the differences between the high-dimensional distributions well.


\subsection{Two-sample Testing on Synthesis Data}\label{sec:synthesis}

Now we investigate the performance of the proposed two-sample test, called the PW test, on high-dimensional synthesis data with $k=3$ across different dimensions $d$ and sample sizes $n=m\in\{75,100\}$.
The target distributions are specified following a similar setup in \cite{reddi2014decreasing}.
In the first case, the distributions $\mu$ and $\nu$ are both $d$-dimensional Gaussian distributions with different mean vectors and the same covariance matrix, where $d\in\{60,120\}$.
More specifically, $\mu=\mathcal{N}(0, I_d)$ and $\nu=\mathcal{N}(\xi, I_d)$, where $\xi=(1,0,\ldots,0)$ is a $d$-dimensional vector with the first entry being one.
In the second case, the distributions $\mu$ and $\nu$ are both $d$-dimensional Gaussian distributions with the same mean vector but different covariance metrics, where $d\in\{30,60\}$.
More specifically, $\mu=\mathcal{N}(0, I_d)$ and $\nu=\mathcal{N}(0,\Sigma)$ with $\Sigma=\mathrm{diag}(4,4,1,\ldots,1)$.
In other words, we only scale the first two diagonal entries in the covariance matrix of $\nu$ to make the hypothesis testing problem difficult to perform.
We compare the performance of the PW test with the MMD test discussed in~\cite{Gretton12}, where the kernel function is chosen to be the standard Gaussian kernel with bandwidth being the empirical median of data points.

The test power for different methods is presented in Fig.~\ref{fig:testing:AUC}, in which the results are averaged for $100$ independent trials.
The first two plots in Fig.~\ref{fig:testing:AUC} show receiver operating characteristic~(ROC) curves for mean-shifted Gaussian distributions, where Fig.~\ref{fig:testing:AUC}a) examines the test power for different choices of sample size $n$ with fixed dimension $d=120$, and Fig.~\ref{fig:testing:AUC}b) examines the power for different choices of $d$ with fixed $n=75$.
The last two plots correspond to covariance-shifted Gaussian distributions, where Fig.~\ref{fig:testing:AUC}c) examines the power for different $n$ with fixed $d=60$, and Fig.~\ref{fig:testing:AUC}d) examines the power for different $d$ with fixed $n=75$.
We can see that the power of all methods increases when the sample size increases, and the power of the PW test is greater than the MMD test especially in high dimensions.
The reason can be explained as follows.
As suggested in \cite{reddi2014decreasing}, the power of MMD test with the median heuristic decays quickly into zero when the dimension $d$ increases.
In contrast, the power of the PW test decreases slower since it operates by projecting high-dimensional data points into a low-dimensional subspace.
Therefore, we can assert that the PW test outperforms the MMD test especially in high dimensions.

\section{Conclusion}
We developed a projected Wasserstein distance for the problem of two-sample testing.
The finite-sample convergence of general IPMs between two empirical distributions was established.
Compared with the Wasserstein distance, the convergence rate of the projected Wasserstein distance has a minor dependence on the dimension of target distributions, which alleviates the curse of dimensionality.
A two-sample test is designed based on this theoretical result, and numerical experiments show that this test outperforms the existing benchmark.
In future work, we will study tighter performance guarantees for the projected Wasserstein distance and develop the optimal choice of $k$ to improve the performance of two-sample tests.


\bibliographystyle{IEEEtran}
\balance
\bibliography{short-bib.bib}

\begin{thebibliography}{10}
\providecommand{\url}[1]{#1}
\csname url@samestyle\endcsname
\providecommand{\newblock}{\relax}
\providecommand{\bibinfo}[2]{#2}
\providecommand{\BIBentrySTDinterwordspacing}{\spaceskip=0pt\relax}
\providecommand{\BIBentryALTinterwordstretchfactor}{4}
\providecommand{\BIBentryALTinterwordspacing}{\spaceskip=\fontdimen2\font plus
\BIBentryALTinterwordstretchfactor\fontdimen3\font minus
  \fontdimen4\font\relax}
\providecommand{\BIBforeignlanguage}[2]{{%
\expandafter\ifx\csname l@#1\endcsname\relax
\typeout{** WARNING: IEEEtran.bst: No hyphenation pattern has been}%
\typeout{** loaded for the language `#1'. Using the pattern for}%
\typeout{** the default language instead.}%
\else
\language=\csname l@#1\endcsname
\fi
#2}}
\providecommand{\BIBdecl}{\relax}
\BIBdecl

\bibitem{Chandola2009}
V.~Chandola, A.~Banerjee, and V.~Kumar, ``Anomaly detection: A survey,''
  \emph{ACM Computing Surveys}, vol.~41, no.~3, Sep. 2009.

\bibitem{bhuyan2013network}
M.~H. Bhuyan, D.~K. Bhattacharyya, and J.~K. Kalita, ``Network anomaly
  detection: methods, systems and tools,'' \emph{IEEE Communications Surveys
  $\&$ Tutorials}, vol.~16, no.~1, pp. 303--336, Jun. 2013.

\bibitem{chandola2010anomaly}
V.~Chandola, A.~Banerjee, and V.~Kumar, ``Anomaly detection for discrete
  sequences: A survey,'' \emph{IEEE Transactions on Knowledge and Data
  Engineering}, vol.~24, no.~5, pp. 823--839, Nov. 2010.

\bibitem{Xie13}
Y.~{Xie}, J.~{Huang}, and R.~{Willett}, ``Change-point detection for
  high-dimensional time series with missing data,'' \emph{IEEE Journal of
  Selected Topics in Signal Processing}, vol.~7, no.~1, pp. 12--27, Dec. 2012.

\bibitem{Shuang15}
S.~Li, Y.~Xie, H.~Dai, and L.~Song, ``M-statistic for kernel change-point
  detection,'' in \emph{Advances in Neural Information Processing Systems},
  Dec. 2015, pp. 3366--3374.

\bibitem{xie2020sequential}
L.~Xie and Y.~Xie, ``Sequential change detection by optimal weighted $\ell_2$
  divergence,'' \emph{IEEE Journal on Selected Areas in Information Theory},
  pp. 1--1, Apr. 2021.

\bibitem{Borgwardt06}
K.~Borgwardt, A.~Gretton, M.~Rasch, H.-P. Kriegel, B.~Schoelkopf, and A.~Smola,
  ``Integrating structured biological data by kernel maximum mean
  discrepancy,'' \emph{Bioinformatics}, vol.~22, pp. 49--57, Jul. 2006.

\bibitem{Schober19}
P.~Schober and T.~Vetter, ``Two-sample unpaired t tests in medical research,''
  \emph{Anesthesia and analgesia}, vol. 129, p. 911, Oct. 2019.

\bibitem{Lloyd15}
J.~R. Lloyd and Z.~Ghahramani, ``Statistical model criticism using kernel two
  sample tests,'' in \emph{Advances in Neural Information Processing Systems},
  vol.~28, Dec. 2015.

\bibitem{chwialkowski2016kernel}
K.~Chwialkowski, H.~Strathmann, and A.~Gretton, ``A kernel test of goodness of
  fit,'' \emph{Proceedings of Machine Learning Research}, vol.~48, pp.
  2606--2615, Jun. 2016.

\bibitem{kim2016examples}
B.~Kim, R.~Khanna, and O.~O. Koyejo, ``Examples are not enough, learn to
  criticize! criticism for interpretability,'' in \emph{Advances in Neural
  Information Processing Systems}, Dec. 2016, pp. 2280--2288.

\bibitem{lehmann2005testing}
E.~L. Lehmann and J.~P. Romano, \emph{Testing statistical hypotheses}, 3rd~ed.,
  ser. Springer Texts in Statistics.\hskip 1em plus 0.5em minus 0.4em\relax
  Springer, 2005.

\bibitem{hotelling1931}
H.~Hotelling, ``The generalization of student's ratio,'' \emph{The Annals of
  Mathematical Statistics}, vol.~2, pp. 360--378, Aug. 1931.

\bibitem{PFANZAGL96}
J.~Pfanzagl and O.~Sheynin, ``Studies in the history of probability and
  statistics xliv a forerunner of the t-distribution,'' \emph{Biometrika},
  vol.~83, no.~4, pp. 891--898, Dec. 1996.

\bibitem{Pratt1981}
J.~W. Pratt and J.~D. Gibbons, \emph{Kolmogorov-Smirnov Two-Sample
  Tests}.\hskip 1em plus 0.5em minus 0.4em\relax Springer New York, 1981.

\bibitem{Frank51}
F.~J.~M. Jr., ``The kolmogorov-smirnov test for goodness of fit,''
  \emph{Journal of the American Statistical Association}, vol.~46, no. 253, pp.
  68--78, Apr. 1951.

\bibitem{Ga1991}
L.~Gy{\"o}rfi and E.~C. Van Der~Meulen, \emph{A Consistent Goodness of Fit Test
  Based on the Total Variation Distance}.\hskip 1em plus 0.5em minus
  0.4em\relax Springer Netherlands, 1991, pp. 631--645.

\bibitem{delbarrio1999}
E.~del Barrio, J.~A. Cuesta-Albertos, C.~Matrain, and J.~M.
  Rodriguez-Rodriguez, ``Tests of goodness of fit based on the
  $l_2$-wasserstein distance,'' \emph{Annals of Statistics}, vol.~27, no.~4,
  pp. 1230--1239, Aug. 1999.

\bibitem{ramdas2015wasserstein}
A.~Ramdas, N.~G. Trillos, and M.~Cuturi, ``On wasserstein two-sample testing
  and related families of nonparametric tests,'' \emph{Entropy}, vol.~19,
  no.~2, Jan. 2017.

\bibitem{Gretton12}
A.~Gretton, K.~M. Borgwardt, M.~J. Rasch, B.~Sch\"{o}lkopf, and A.~Smola, ``A
  kernel two-sample test,'' \emph{Journal of Machine Learning Research},
  vol.~13, pp. 723--773, Mar. 2012.

\bibitem{Gretton09}
A.~Gretton, K.~Fukumizu, Z.~Harchaoui, and B.~K. Sriperumbudur, ``A fast,
  consistent kernel two-sample test,'' in \emph{Advances in Neural Information
  Processing Systems}, Dec. 2009, pp. 673--681.

\bibitem{Grettonnips12}
A.~Gretton, D.~Sejdinovic, H.~Strathmann, S.~Balakrishnan, M.~Pontil,
  K.~Fukumizu, and B.~K. Sriperumbudur, ``Optimal kernel choice for large-scale
  two-sample tests,'' in \emph{Advances in Neural Information Processing
  Systems}, Dec. 2012, pp. 1205--1213.

\bibitem{reddi2014decreasing}
A.~Ramdas, S.~J. Reddi, B.~P\'{o}czos, A.~Singh, and L.~Wasserman, ``On the
  decreasing power of kernel and distance based nonparametric hypothesis tests
  in high dimensions,'' in \emph{Proceedings of the Twenty-Ninth AAAI
  Conference on Artificial Intelligence}, Jan. 2015.

\bibitem{Bigot2017CentralLT}
J.~Bigot, E.~Cazelles, and N.~Papadakis, ``Central limit theorems for sinkhorn
  divergence between probability distributions on finite spaces and statistical
  applications,'' \emph{Electronic Journal of Statistics}, Dec. 2017.

\bibitem{kanamori2010fdivergence}
T.~{Kanamori}, T.~{Suzuki}, and M.~{Sugiyama}, ``$f$-divergence estimation and
  two-sample homogeneity test under semiparametric density-ratio models,''
  \emph{IEEE Transactions on Information Theory}, vol.~58, no.~2, pp. 708--720,
  Sep. 2011.

\bibitem{kim2020classification}
I.~Kim, A.~Ramdas, A.~Singh, and L.~Wasserman, ``{Classification accuracy as a
  proxy for two-sample testing},'' \emph{The Annals of Statistics}, vol.~49,
  no.~1, pp. 411 -- 434, Feb. 2021.

\bibitem{Jolliffe1986}
I.~Jolliffe, \emph{Principal Component Analysis}.\hskip 1em plus 0.5em minus
  0.4em\relax Springer Verlag, 1986.

\bibitem{mclachlan1992discriminant}
G.~McLachlan and J.~W.~. Sons, \emph{Discriminant Analysis and Statistical
  Pattern Recognition}, ser. Wiley Series in Probability and Statistics.\hskip
  1em plus 0.5em minus 0.4em\relax Wiley, 1992.

\bibitem{wei2013directionprojectionpermutation}
S.~Wei, C.~Lee, L.~Wichers, G.~Li, and J.~S. Marron,
  ``Direction-projection-permutation for high dimensional hypothesis tests,''
  \emph{Journal of Computational and Graphical Statistics}, vol.~25, no.~2, pp.
  549--569, May 2016.

\bibitem{Ghosh16}
A.~K. Ghosh and M.~Biswas, ``Distribution-free high-dimensional two-sample
  tests based on discriminating hyperplanes,'' \emph{TEST}, vol.~25, no.~3, pp.
  525--547, Dec. 2015.

\bibitem{mueller2015principal}
J.~W. Mueller and T.~Jaakkola, ``Principal differences analysis: Interpretable
  characterization of differences between distributions,'' in \emph{Advances in
  Neural Information Processing Systems}, vol.~28, Dec. 2015.

\bibitem{lin2020projection}
T.~Lin, C.~Fan, N.~Ho, M.~Cuturi, and M.~Jordan, ``Projection robust
  wasserstein distance and riemannian optimization,'' in \emph{Advances in
  Neural Information Processing Systems}, vol.~33, Dec. 2020, pp. 9383--9397.

\bibitem{lin2020projection2}
T.~Lin, Z.~Zheng, E.~Chen, M.~Cuturi, and M.~Jordan, ``On projection robust
  optimal transport: Sample complexity and model misspecification,'' in
  \emph{Proceedings of The 24th International Conference on Artificial
  Intelligence and Statistics}, vol. 130, Apr. 2021, pp. 262--270.

\bibitem{huang2021riemannian}
M.~Huang, S.~Ma, and L.~Lai, ``A riemannian block coordinate descent method for
  computing the projection robust wasserstein distance,'' \emph{arXiv preprint
  arXiv:2012.05199}, 2021.

\bibitem{cuturi2013sinkhorn}
M.~Cuturi, ``Sinkhorn distances: Lightspeed computation of optimal transport,''
  in \emph{Advances in Neural Information Processing Systems}, Dec. 2013, p.
  2292–2300.

\bibitem{genevay2019sample}
A.~Genevay, L.~Chizat, F.~Bach, M.~Cuturi, and G.~Peyre, ``Sample complexity of
  sinkhorn divergences,'' in \emph{Proceedings of the 22rd International
  Conference on Artificial Intelligence and Statistics}, vol.~89, Apr. 2019,
  pp. 1574--1583.

\bibitem{Bonneel14}
N.~Bonneel, J.~Rabin, G.~PeyrA, and H.~Pfister, ``Sliced and radon wasserstein
  barycenters of measures,'' \emph{Journal of Mathematical Imaging and Vision},
  vol.~51, Apr. 2014.

\bibitem{deshp2019maxsliced}
I.~Deshpande, Y.-T. Hu, R.~Sun, A.~Pyrros, N.~Siddiqui, S.~Koyejo, Z.~Zhao,
  D.~Forsyth, and A.~G. Schwing, ``Max-sliced wasserstein distance and its use
  for gans,'' in \emph{2019 IEEE/CVF Conference on Computer Vision and Pattern
  Recognition (CVPR)}, Jun. 2019, pp. 10\,640--10\,648.

\bibitem{manole2019minimax}
T.~Manole, S.~Balakrishnan, and L.~Wasserman, ``Minimax confidence intervals
  for the sliced wasserstein distance,'' \emph{arXiv preprint
  arXiv:1909.07862}, 2019.

\bibitem{villani2016optimal}
C.~Villani, \emph{Optimal Transport: Old and New}, ser. Grundlehren der
  mathematischen Wissenschaften.\hskip 1em plus 0.5em minus 0.4em\relax
  Springer Berlin Heidelberg, 2016.

\bibitem{zhang2018discriminationgeneralization}
P.~Zhang, Q.~Liu, D.~Zhou, T.~Xu, and X.~He, ``On the
  discrimination-generalization tradeoff in gans,'' in \emph{6th International
  Conference on Learning Representations, {ICLR}}, Feb. 2018.

\bibitem{Luxburg04}
U.~v. Luxburg and O.~Bousquet, ``Distance-based classification with lipschitz
  functions,'' \emph{Journal of Machine Learning Research}, vol.~5, pp.
  669--695, Jun. 2004.

\bibitem{Subspacepaty}
F.-P. Paty and M.~Cuturi, ``Subspace robust {W}asserstein distances,'' in
  \emph{Proceedings of the 36th International Conference on Machine Learning},
  vol.~97, Jun. 2019, pp. 5072--5081.

\bibitem{mcdiarmid_1989}
C.~McDiarmid, \emph{On the method of bounded differences}, ser. London
  Mathematical Society Lecture Note Series.\hskip 1em plus 0.5em minus
  0.4em\relax Cambridge University Press, 1989.

\bibitem{maurer2016vectorcontraction}
A.~Maurer, ``A vector-contraction inequality for rademacher complexities,'' in
  \emph{Algorithmic Learning Theory}.\hskip 1em plus 0.5em minus 0.4em\relax
  Springer International Publishing, Sep. 2016, pp. 3--17.

\bibitem{Gin15}
E.~Gin and R.~Nickl, \emph{Mathematical Foundations of Infinite-Dimensional
  Statistical Models}, 1st~ed.\hskip 1em plus 0.5em minus 0.4em\relax USA:
  Cambridge University Press, 2015.

\end{thebibliography}

\ifnum\paperversion=2
\appendices
\clearpage
\onecolumn
\section{Proof of Technical Results}\label{Sec:appendix:proof}

\begin{lemma}[McDiarmid's Inequality~\cite{mcdiarmid_1989}]\label{Lemma:original:Mc}
Let $X_1,\ldots,X_n$ be independent random variables.
Assume that the function $f:~\prod_{i=1}^n\mathcal{X}_i\to\mathbb{R}$ satisfies the bounded difference property:
\[
\sup_{\substack{x_i\in\mathcal{X}_i, i=1,2,\ldots,n,\\\tilde{x}_i\in\mathcal{X}_i}} |f(x_1,\ldots,x_m) - f(x_1,\ldots,x_{i-1}, \tilde{x}_i, x_{i+1},\ldots,x_m)|\le c_i.
\]
Then we have
\[
\mathbb{P}\left(
|f - \mathbb{E}f|>t
\right)
\le
2\exp\left(
-\frac{2t^2}{\sum_{i=1}^nc_i^2}
\right).
\]
\end{lemma}

\begin{IEEEproof}[Proof of Proposition~\ref{Proposition:finite:IPM}]
\begin{itemize}
\item
Step 1: Simplify $|\text{IPM}({\mu},{\nu}) - \text{IPM}(\hat{{\mu}}_n,\hat{{\nu}}_m)|$:
\begin{align*}
&|\text{IPM}({\mu},{\nu}) - \text{IPM}(\hat{{\mu}}_n,\hat{{\nu}}_m)|\\
=&\left|
\sup_{f\in\mathcal{F}}
\big(
\mathbb{E}_{x\sim {\mu}}[f(x)]
-
\mathbb{E}_{y\sim {\nu}}[f(y)]
\big)
-
\sup_{f\in\mathcal{F}}
\left(
\frac{1}{n}\sum_{i=1}^nf(x_i)
-
\frac{1}{m}\sum_{i=1}^mf(y_i)
\right)
\right|\\
\le&
\underbrace{\sup_{f\in\mathcal{F}}
\left|
\mathbb{E}_{x\sim {\mu}}[f(x)]
-
\mathbb{E}_{y\sim {\nu}}[f(y)]
-
\frac{1}{n}\sum_{i=1}^nf(x_i)
+
\frac{1}{m}\sum_{i=1}^mf(y_i)
\right|}_{\Delta({\mu},{\nu},\{x_i\},\{y_i\})}
\end{align*}
\item
Step 2: Bound the concentration term (this follows the similar argument in \cite[A.2]{Gretton12}):
\begin{align*}
&\mathbb{E}_{x_i\sim {\mu}, y_i\sim{\nu}}[\Delta({\mu},{\nu},\{x_i\},\{y_i\})]\\
=&
\mathbb{E}_{x_i\sim {\mu}, y_i\sim{\nu}}\left[\sup_{f\in\mathcal{F}}
\left|
\mathbb{E}_{x\sim {\mu}}[f(x)]
-
\mathbb{E}_{y\sim {\nu}}[f(y)]
-
\frac{1}{n}\sum_{i=1}^nf(x_i)
+
\frac{1}{m}\sum_{i=1}^mf(y_i)
\right|\right]\\
=&
\mathbb{E}_{x_i\sim {\mu}, y_i\sim{\nu}}\left[\sup_{f\in\mathcal{F}}
\left|
\mathbb{E}_{x_i'\sim {\mu}}
\left(
\frac{1}{n}\sum_{i=1}^nf(x_i')
\right)
-
\mathbb{E}_{y_i'\sim {\nu}}
\left(
\frac{1}{m}\sum_{i=1}^mf(y_i')
\right)
-
\frac{1}{n}\sum_{i=1}^nf(x_i)
+
\frac{1}{m}\sum_{i=1}^mf(y_i)
\right|\right]\\
\le&
\mathbb{E}_{\substack{x_i\sim {\mu}, y_i\sim{\nu} \\ x_i'\sim\mu, y_i'\sim\nu}}
\left[
\sup_{f\in\mathcal{F}}
\left|
\frac{1}{n}\sum_{i=1}^nf(x_i')
-
\frac{1}{n}\sum_{i=1}^nf(x_i)
-
\frac{1}{m}\sum_{i=1}^mf(y_i')
+
\frac{1}{m}\sum_{i=1}^mf(y_i)
\right|
\right]\\
=&
\mathbb{E}_{\substack{x_i\sim {\mu}, y_i\sim{\nu} \\ x_i'\sim\mu, y_i'\sim\nu\\ \sigma, \sigma'}}
\left[
\sup_{f\in\mathcal{F}}
\left|
\frac{1}{n}\sum_{i=1}^n
\sigma_i
\left(
f(x_i') - f(x_i)
\right)
+
\frac{1}{m}\sum_{i=1}^m
\sigma_i'
\left(
f(y_i') - f(y_i)
\right)
\right|
\right]\\
\le&
\mathbb{E}_{x_i\sim \mu, x_i'\sim {\mu},\sigma}
\left[
\sup_{f\in\mathcal{F}}
\left|
\frac{1}{n}\sum_{i=1}^n
\sigma_i
\left(
f(x_i') - f(x_i)
\right)
\right|
\right]
+
\mathbb{E}_{y_i\sim \nu, y_i'\sim {\nu},\sigma'}
\left[
\sup_{f\in\mathcal{F}}
\left|
\frac{1}{m}\sum_{i=1}^m
\sigma_i'
\left(
f(y_i') - f(y_i)
\right)
\right|
\right]\\
\le&
2[
\mathfrak{R}_n(\mathcal{F}, {\mu})
+
\mathfrak{R}_m(\mathcal{F}, {\nu})
]
\end{align*}
It follows that
\begin{align*}
&\mathbb{P}\bigg\{
\big|\text{IPM}({\mu},{\nu}) - \text{IPM}(\hat{{\mu}}_n,\hat{{\nu}}_m)\big|
>
\epsilon
+2[
\mathfrak{R}_n(\mathcal{F}, {\mu})
+
\mathfrak{R}_m(\mathcal{F}, {\nu})
]
\bigg\}\\
\le&
\mathbb{P}\bigg\{
\big|\Delta({\mu},{\nu},\{x_i\},\{y_i\})
-
\mathbb{E}_{x_i\sim {\mu}, y_i\sim{\nu}}[\Delta({\mu},{\nu},\{x_i\},\{y_i\})]
\big|
>
\epsilon
\big]
\bigg\}
\end{align*}
\item
Step 3: Apply the concentration inequality:
consider the function 
\[
(x_1,\ldots,x_n,y_1,\ldots,y_m)\mapsto \Delta({\mu},{\nu},\{x_i\}_{i=1}^n,\{y_j\}_{j=1}^m).
\]
We can see that 
\begin{align*}
&\left|
\Delta(x_1,\ldots,x_n,y_1,\ldots,y_m) - \Delta(x_1,\ldots,x_{i-1},\tilde{x}_i,x_{i+1},x_n,y_1,\ldots,y_m)
\right|\\
\le&\frac{1}{n}\sup_{f\in\mathcal{F}}
|f(x_i) - f(\tilde{x_i})|
\\\le&\frac{L}{n}d(x_i,\tilde{x}_i)\le \frac{B_{\mu}L}{n}.
\end{align*}
Similarly, 
\[
\left|
\Delta(x_1,\ldots,x_n,y_1,\ldots,y_m) - \Delta(x_1,\ldots,x_n,y_1,\ldots,y_{i-1},\tilde{y}_i,y_{i+1},\ldots,y_m)
\right|
\le \frac{B_{\nu}L}{m}.
\]
Hence, applying Lemma~\ref{Lemma:original:Mc} with $c_i=\frac{B_\mu L}{n}, i=1,\ldots,n$ and $c_i=\frac{B_\nu L}{m}, i=n+1,\ldots,n+m$ gives the desired result:
\begin{align*}
&\mathbb{P}\bigg\{
\big|\Delta({\mu},{\nu},\{x_i\},\{y_i\})
-
\mathbb{E}_{x_i\sim {\mu}, y_i\sim{\nu}}[\Delta({\mu},{\nu},\{x_i\},\{y_i\})]
\big|
>
\epsilon
\big]
\bigg\}\\
\le&
2\exp\left(
-\frac{2mn\epsilon^2}{L^2(mB_{\mu}^2 + nB_{\nu}^2)}
\right).
\end{align*}

\end{itemize}
\end{IEEEproof}

\begin{IEEEproof}[Proof of Proposition~\ref{Proposition:special:P1}]
In this case, we have
\[
\big|\text{IPM}({\mu},{\nu}) - \text{IPM}(\hat{{\mu}}_n,\hat{{\nu}}_m)\big|
\le
\underbrace{\sup_{f\in\mathcal{F}}
\left|
\frac{1}{n}\sum_{i=1}^nf(x_i)
-
\frac{1}{n}\sum_{i=1}^nf(y_i)
\right|}_{\Delta(\{x_i\},\{y_i\})},
\]
where 
\begin{align*}
\mathbb{E}_{x_i\sim\mu, y_i\sim\mu}
[\Delta(\{x_i\},\{y_i\})]
&=
\mathbb{E}_{x_i\sim\mu, y_i\sim\mu}
\left[
\sup_{f\in\mathcal{F}}
\left|
\frac{1}{n}\sum_{i=1}^nf(x_i) - \frac{1}{n}\sum_{i=1}^nf(y_i)
\right|
\right]\\
&\le \mathbb{E}_{x_i\sim\mu, y_i\sim\mu, \sigma}
\left[
\sup_{f\in\mathcal{F}}
\left|
\frac{1}{n}\sum_{i=1}^n\sigma_i(f(x_i) - f(y_i))
\right|
\right]
\le2\mathfrak{R}_n(\mathcal{F}, {\mu}).
\end{align*}
We can see that 
\begin{align*}
\left|
\Delta(x_1,\ldots,x_n,y_1,\ldots,y_m) - \Delta(x_1,\ldots,x_{i-1},\tilde{x}_i,x_{i+1},x_n,y_1,\ldots,y_m)
\right|&\le \frac{B_{\mu}L}{n},\\
\left|
\Delta(x_1,\ldots,x_n,y_1,\ldots,y_m) - \Delta(x_1,\ldots,x_n,y_1,\ldots,y_{i-1},\tilde{y}_i,y_{i+1},\ldots,y_m)
\right|&\le \frac{B_{\mu}L}{n}.
\end{align*}

Hence, applying Lemma~\ref{Lemma:original:Mc} with $c_i=\frac{B_{\mu}L}{n}, i=1,\ldots,n+m$ gives the desired result.
\end{IEEEproof}

\begin{lemma}\label{Lemma:bias:Rademacher}
For fixed $g\in\Lip_1$, we have that 
\[
\mathbb{E}_{x_i\sim{\mu}, \sigma_i}
\left[ 
\sup_{A: A\trans A=I} \frac{1}{n}\sum_{i=1}^n\sigma_ig(A\trans x_i)
\right]
\le 
\frac{\sqrt{2k}}{\sqrt{n}}
\mathbb{E}_{x\sim{\mu}}\sqrt{\|x\|^2}
\]
\end{lemma}
\begin{IEEEproof}[Proof of Lemma~\ref{Lemma:bias:Rademacher}]
Denote by $a_j$ the $j$-th column of the matrix $A$. 
Then we have
\begin{subequations}
\begin{align}
&\mathbb{E}_{x_i\sim{\mu}, \sigma_i}
\left[ 
\sup_{A: A\trans A=I} \frac{1}{n}\sum_{i=1}^n\sigma_ig(A\trans x_i)
\right]\nonumber\\
\le &
\frac{\sqrt{2}}{n}
\mathbb{E}_{x_i\sim{\mu}, \sigma_i}\left[
\sup_{A: A\trans A=I}\sum_{i=1}^n\sum_{j=1}^k \sigma_{i,j}a_j\trans x_i
\right]\label{Eq:proof:rade:a}\\
\le &
\frac{\sqrt{2}}{n}
\mathbb{E}_{x_i\sim{\mu}, \sigma_i}\left[
\sup_{A: A\trans A=I}\sum_{j=1}^k\|a_j\|_2
\left\|
\sum_{i=1}^n\sigma_{i,j}x_i
\right\|_2
\right]\label{Eq:proof:rade:b}\\
\le &
\frac{\sqrt{2}}{n}
\mathbb{E}_{x_i\sim{\mu}, \sigma_i}\sqrt{
\sum_{j=1}^k
\left\|
\sum_{i=1}^n\sigma_{i,j}x_i
\right\|_2^2
}\nonumber\\
\le &\frac{\sqrt{2}}{n}
\mathbb{E}_{x_i\sim{\mu}}\sqrt{
k
\sum_{i=1}^n\left\|
x_i
\right\|_2^2
}\nonumber\\
= &\frac{\sqrt{2k}}{\sqrt{n}}
\mathbb{E}_{x\sim{\mu}}\sqrt{\|x\|^2}\nonumber
\end{align}
\end{subequations}
where \eqref{Eq:proof:rade:a} is by applying Corollary~4 in \cite{maurer2016vectorcontraction},
and \eqref{Eq:proof:rade:b} is by applying the Cauchy-Schwartz inequality.
\end{IEEEproof}

\begin{lemma}\label{Lemma:Rademacher:variation}
Denote by $D_{\mu} = \max_{x}\|x\|$. Then we have that
\begin{align*}
&\mathbb{E}_{x_i\sim{\mu}, \sigma_i}
\left[ 
\sup_{g\in\Lip_1}
\left(
\sup_{A: A\trans A=I} \frac{1}{n}\sum_{i=1}^n\sigma_ig(A\trans x_i)
-
\mathbb{E}_{x_i\sim{\mu}, \sigma_i}
\sup_{A: A\trans A=I} \frac{1}{n}\sum_{i=1}^n\sigma_ig(A\trans x_i)
\right)
\right]\\
&\qquad\qquad \le\inf_{\epsilon>0}\left\{
2\epsilon + \sqrt{\frac{36}{n}(B_{\mu}^2+4D_{\mu}^2)}\cdot \sqrt{\log\left( 
2\left\lceil 
\frac{2}{\epsilon}
\right\rceil+1
\right) + (1+2/\epsilon)^k\log 2}
\right\}.
\end{align*}
\end{lemma}
\begin{IEEEproof}[Proof of Lemma~\ref{Lemma:Rademacher:variation}]
Define the empirical process
\[
X_{g} = \sup_{A: A\trans A=I} \frac{1}{n}\sum_{i=1}^n\sigma_ig(A\trans x_i)
-
\mathbb{E}_{x_i\sim{\mu}, \sigma_i}
\sup_{A: A\trans A=I} \frac{1}{n}\sum_{i=1}^n\sigma_ig(A\trans x_i).
\]
It is easy to see that $\mathbb{E}[X_{g}]=0$. 
We first claim that $X_g$ is a sub-Gaussian variable for fixed $g\in\Lip_1$.
Consider $Z=\{x_i\}_{i=1}^n$ and $\kappa=\{\sigma_i\}_{i=1}^n$ and define $f(Z,\kappa)=\sup_{A: A\trans A=I} \frac{1}{n}\sum_{i=1}^n\sigma_ig(A\trans x_i)$. 
Then
\begin{align*}
|f(Z,\kappa) - f(Z_{(i)}',\kappa)&\le 
\left|
\frac{1}{n}\sup_{A: A\trans A=I}\sigma_i[g(A\trans x_i) - g(A\trans x_i')]
\right|
\le \frac{1}{n}\sup_{A: A\trans A=I}\|A\trans(x_i-x_i')\|\le \frac{B_{\mu}}{n}\\
|f(Z,\kappa) - f(Z,\kappa_{(i)}')&\le 
\left|
\frac{1}{n}\sup_{A: A\trans A=I}(\sigma_i - \sigma_i')g(A\trans x_i)
\right|
\le 
\frac{2}{n}\sup_{A:~A\trans A=I}\|A\trans x_i\|\le \frac{2}{n}D_{\mu},
\end{align*}
with $D_{\mu}=\max_{x}\|x\|$.
Therefore, applying McDiarmid's inequality in Lemma~\ref{Lemma:original:Mc} implies that
\begin{align*}
\text{Pr}
\left\{
|X_g|\ge u
\right\}
=
\text{Pr}
\left\{
|f(Z,\kappa) - \mathbb{E}[f(Z,\kappa)]|\ge u
\right\}
\le 2\exp\left(
-\frac{2nt^2}{B_{\mu}^2 + 4D_{\mu}^2}
\right).
\end{align*}
which means that for fixed $g$, the random variable $X_g$ is sub-Gaussian~\cite[Lemma 2.3.2]{Gin15} with the parameter $\frac{18}{n}(B_{\mu}^2+4D_{\mu}^2)$.
Define the metric $\textsf{d}(g,g')=\sup_{x:~\|x\|_2\le D_{\mu}}|g(x) - g'(x)|$.
For any $g,g'\in\Lip_1$, we have that
\begin{align*}
    |X_g-X_{g'}|&\le \left|\sup_{A: A\trans A=I}\bigg(
    \frac{1}{n}\sum_{i=1}^n\sigma_i[g(A\trans x_i)-g'(A\trans x_i)]
    \bigg)\right|\\
    &\qquad + \left|\mathbb{E}_{x_i\sim\mu, \sigma_i}\left[ 
    \sup_{A: A\trans A=I}\bigg(
    \frac{1}{n}\sum_{i=1}^n\sigma_i[g(A\trans x_i)-g'(A\trans x_i)]
    \bigg)
    \right]\right|\\
    &\le 2\sup_{A: A\trans A=I}\sup_{x}|g(A\trans x) - g'(A\trans x)|\le 2\sup_{y:~\|y\|\le D_{\mu}}|g(y) - g'(y)|\\
    &=2\textsf{d}(g,g').
\end{align*}
Thus, by the standard $\epsilon$-net argument, we have that 
\[
\mathbb{E}_{x_i\sim{\mu}, \sigma_i}\left[ 
\sup_{g\in Lip_1}
X_g
\right]
\le \inf_{\epsilon>0}
\left\{
2\epsilon + \sqrt{\frac{18}{n}(B_{\mu}^2+4D_{\mu}^2)}\cdot \sqrt{2\log\mathcal{N}(\Lip_1, \epsilon, \textsf{d})}
\right\}
\]
Then the covering number 
\[
\mathcal{N}(\epsilon, \Lip_1, \textsf{d}) = \left( 
2\left\lceil 
\frac{2B_{\mu}}{\epsilon}
\right\rceil+1
\right)\cdot 2^{(1+2B_{\mu}/\epsilon)^k}
\]
which implies that
\[
\mathbb{E}_{x_i\sim{\mu}, \sigma_i}\left[ 
\sup_{g\in Lip_1}
X_g
\right]
\le \inf_{\epsilon>0}
\left\{
2\epsilon + \sqrt{\frac{36}{n}(B_{\mu}^2+4D_{\mu}^2)}\cdot \sqrt{\log\left( 
2\left\lceil 
\frac{2B_{\mu}}{\epsilon}
\right\rceil+1
\right) + (1+2B_{\mu}/\epsilon)^k\log 2}
\right\}.
\]

\end{IEEEproof}

\begin{IEEEproof}[Proof of Proposition~\ref{Proposition:Rademacher}]
The Rademacher complexity $\mathfrak{R}_n(\mathcal{F}, {\mu})$ admits the following bias-variation decomposition:
\begin{align*}
\mathfrak{R}_n(\mathcal{F}, {\mu})
&\le 
\sup_{g\in\Lip_1} \mathbb{E}_{x_i\sim{\mu}, \sigma_i}
\left[ 
\sup_{A: A\trans A=I} \frac{1}{n}\sum_{i=1}^n\sigma_ig(A\trans x_i)
\right]
\\&+
\mathbb{E}_{x_i\sim{\mu}, \sigma_i}
\left[ 
\sup_{g\in\Lip_1}
\left(
\sup_{A: A\trans A=I} \frac{1}{n}\sum_{i=1}^n\sigma_ig(A\trans x_i)
-
\mathbb{E}_{x_i\sim{\mu}, \sigma_i}
\sup_{A: A\trans A=I} \frac{1}{n}\sum_{i=1}^n\sigma_ig(A\trans x_i)
\right)
\right]
\end{align*}
Applying the result in Lemma~\ref{Lemma:bias:Rademacher} and Lemma~\ref{Lemma:Rademacher:variation} upper bounds the bias and variation term, respectively.

\end{IEEEproof}

\section{Details for Computing the Projected Wasserstein Distance}\label{Appendix:sec:computation}
Denote the space of orthogonal projectors as $\mathcal{M}=\{A\in\mathbb{R}^{d\times k}:~A\trans A=I_k\}$.
For fixed $A\in\mathcal{M}$, define the mapping
\begin{equation}\label{Eq:f:A}
    \begin{aligned}
f(A)&=\min_{\pi\in\Gamma(\hat{\mu}_n,\hat{\nu}_m)}~\int \textsf{d}(A\trans x,A\trans y)\diff\pi(x,y)\\
&=\min_{\pi\in\mathbb{R}^{n\times m}_+}~\left\{\sum_{(i,j)}\pi_{i,j}\|A\trans(x_i-y_j)\|_2:\quad \sum_i\pi_{i,j}=\frac{1}{m}, \sum_{j}\pi_{i,j}=\frac{1}{n}\right\}.
\end{aligned}
\end{equation}
Based on Definition~\ref{Definition:projected:wasserstein:distance}, the computation of the projected Wasserstein distance $\mathcal{P}W(\hat{\mu}_n, \hat{\nu}_m)$ can be formulated as the manifold optimization problem
\begin{equation}
    \max_{A\in\mathcal{M}}~f(A).\label{Eq:manifold:optimization}
\end{equation}
Here we use the Riemannian gradient method to obtain a near-optimal solution to the problem defined above.
Suppose that $\pi^*$ is an optimal solution to \eqref{Eq:f:A}.
Based on the Danskin's theorem, the sub-differential of the objective function is given by
\begin{align*}
\nabla f(A)&=\sum_{i,j}\pi^*_{i,j}\nabla~\|A\trans(x_i-y_j)\|_2.
\end{align*}
Since the $\ell_2$ norm is a non-smooth function, we approximate it with a smooth function $\|x\|_{2,\kappa}=\sqrt{x^2+\kappa^2}-\kappa$. 
Then an approximated sub-differential can be computed as 
\begin{equation*}
\begin{aligned}
\widehat{\nabla} f(A)&=\sum_{i,j}\pi^*_{i,j}\nabla~\|A\trans(x_i-y_j)\|_{2,\kappa}\\
&=\left[ \sum_{i,j}\frac{\pi^*_{i,j}}{\|A\trans(x_i-y_j)\|_{2,\kappa}}
(x_i-y_j)(x_i-y_j)\trans\right]A.
\end{aligned}
\end{equation*}
Moreover, the computation of this approximated sub-differential can be arranged by avoiding double-loop summation over indexes $i$ and $j$.
Then we can obtain the approximated Riemannian sub-differential of $f$ by the operation
\[
\text{Grad}[f](A) = \mathcal{P}_{T_A}\big(\nabla f(A)\big),
\]
where $\mathcal{P}_{T_A}$ denotes the projection operator into the tangent space at $A$, which is specified as
\[
\mathcal{P}_{T_A}(G) = G - A\frac{G\trans A + A\trans G}{2}.
\]
A generic Riemannian gradient method solves \eqref{Eq:manifold:optimization} based on the iteration
\[
A_{t+1} \leftarrow \text{Retr}_{A_t}(\gamma_{t+1}\zeta_{t+1}),
\]
where $\zeta_{t+1}$ is a Riemannian sub-differential of $f$ at $A_t$, $\text{Retr}$ is any retraction on Stiefel manifold, and $\gamma_{t+1}=\gamma_0/\sqrt{t+1}$ stands for the diminishing step size.
We summarize the pseudocode of the Riemannian gradient method in Algorithm~\ref{alg:1}.

\begin{algorithm}[ht]
\caption{Riemannian Gradient Method for Solving \eqref{Eq:manifold:optimization}}
\label{alg:1} 
\begin{algorithmic}[1] 
\REQUIRE
{
Data points $\{x_i\}_{i=1}^n$ and $\{y_i\}_{i=1}^m$,
step sizes $\{\gamma_t\}_{t=0}^{T}$,
tolerance $\epsilon$
}
\STATE{Initialize $A_0\in\mathcal{M}$}
\FOR{$t=0,1,\ldots,T-1$}
\STATE{
Compute the optimal transport solution between two set of points $\{A\trans x_i\}_{i=1}^n$ and $\{A\trans y_i\}_{i=1}^m$, denoted as $\pi_{t+1}$.
}
\STATE{
Compute $\zeta_{t+1}\leftarrow \text{Grad}[f](A_t)$.
}
\STATE{
Compute $A_{t+1}\leftarrow \text{Retr}_{A_t}(\gamma_{t+1}\zeta_{t+1})$.
}
\STATE{
Stop the iteration when $\|A_{t+1}-A_t\|_F/(1 + \|A_t\|_F)\le\epsilon$.
}
\ENDFOR
\end{algorithmic}
\end{algorithm}
\fi

\end{document}